\def\year{2022}\relax
\documentclass[letterpaper]{article} 
\usepackage{aaai22}  
\usepackage{times}  
\usepackage{helvet}  
\usepackage{courier}  
\usepackage[hyphens]{url}  
\usepackage{graphicx} 
\usepackage{natbib}  
\usepackage{caption} 
\DeclareCaptionStyle{ruled}{labelfont=normalfont,labelsep=colon,strut=off} 
\usepackage{algorithm}
\usepackage{algorithmic}
\usepackage{amsmath}
\usepackage{amsthm}
\usepackage{amssymb}
\newtheorem{theorem}{Theorem}
\newtheorem{lemma}[theorem]{Lemma}
\usepackage{multirow}
\usepackage{diagbox}
\usepackage{boldline} 
\usepackage{newfloat}
\usepackage{listings}
\floatstyle{ruled}
\newfloat{listing}{tb}{lst}{}
\floatname{listing}{Listing}

\nocopyright
\title{Can Graph Neural Networks Learn to Solve MaxSAT Problem?}
\author{
    Minghao Liu\textsuperscript{\rm 1,3},
    Fuqi Jia\textsuperscript{\rm 1,3},
    Pei Huang\textsuperscript{\rm 1,3},
    Fan Zhang\textsuperscript{\rm 5},
    Yuchen Sun\textsuperscript{\rm 4},\\
    Shaowei Cai\textsuperscript{\rm 1,3},
    Feifei Ma\textsuperscript{\rm 1,2,3},
    Jian Zhang\textsuperscript{\rm 1,3}
}
\affiliations{
    \textsuperscript{\rm 1}State Key Laboratory of Computer Science,\\Institute of Software, Chinese Academy of Sciences\\
    \textsuperscript{\rm 2}Laboratory of Parallel Software and Computational Science,\\Institute of Software, Chinese Academy of Sciences\\
    \textsuperscript{\rm 3}University of Chinese Academy of Sciences\\
    \textsuperscript{\rm 4}Inspir.ai \textsuperscript{\rm 5}Individual researcher \\
    Beijing, China\\

    \{liumh, maff\}@ios.ac.cn
}

\usepackage{bibentry}

\begin{document}

\maketitle

\begin{abstract}
	With the rapid development of deep learning techniques, various recent work has tried to apply graph neural networks (GNNs) to solve NP-hard problems such as Boolean Satisfiability (SAT), which shows the potential in bridging the gap between machine learning and symbolic reasoning. However, the quality of solutions predicted by GNNs has not been well investigated in the literature. In this paper, we study the capability of GNNs in learning to solve Maximum Satisfiability (MaxSAT) problem, both from theoretical and practical perspectives. We build two kinds of GNN models to learn the solution of MaxSAT instances from benchmarks, and show that GNNs have attractive potential to solve MaxSAT problem through experimental evaluation. We also present a theoretical explanation of the effect that GNNs can learn to solve MaxSAT problem to some extent for the first time, based on the algorithmic alignment theory.
\end{abstract}

\section{Introduction}
The propositional logic has long been recognized as one of the corner stones of reasoning in philosophy and mathematics \cite{biere2009handbook}. The satisfiability problem of propositional logic formulas (SAT) has significant impact on many areas of computer science and artificial intelligence. SAT is the first problem proved to be NP-complete \cite{cook1971complexity}, requiring the worst-case exponential time to be solved unless P=NP. The Maximum Satisfiability problem (MaxSAT) is a generalization of SAT, which aims to optimize the number of satisfied clauses, and it is also an NP-hard problem with a wide range of applications \cite{carlos2013maxsat}. There have been tremendous efforts in solving SAT and MaxSAT efficiently, and numerous solvers have been developed in the last few decades such as \cite{cai2021sat}. Nevertheless, these off-the-shelf solvers are mainly based on search algorithms with elaborate hand-crafted strategies. Thanks to recent advances in deep representation learning, particularly for non-euclidean structure, there have been initial efforts to represent and solve combinatorial problems such as SAT through data-driven approaches. It relies on the fact that the distribution of problem instances in practice is generally domain-specific, which means that we can solely replace the training data to obtain efficient implicit heuristics instead of carefully modifying the strategies for certain domain.

The current application of deep learning to SAT problem can be grouped into two categories. One is to build an end-to-end model that inputs a problem instance and directly predicts the satisfiability or solution, as demonstrated in Figure \ref{fig1}. A pioneering work is NeuroSAT \cite{selsam2018learning}, which shows that graph neural networks (GNNs) have capability to learn the satisfiability of SAT instances in specific domain (i.e., whether a group of assignments to the variables exists such that all the clauses are satisfied). \cite{chris2020predicting} further explores the performance of GNN models on predicting the satisfiability of random 3SAT problems which are still challenging for the state-of-the-art solvers. The other category combines neural networks with the traditional search frameworks, trying to improve the effects of some key heuristics. For example,  GNNs are used to replace the variable selection function in the WalkSAT solver \cite{nips2019learning},  produce the initialization assignment of local search \cite{zhang2020nlocalsat}, or guide the search by predicting the variables in the unsat core \cite{selsam2019neurocore}. These efforts indicate that GNN models are believed to have a certain degree of capability for learning from SAT problem instances, and also the potential to help improve SAT solving techniques in the future.

Nevertheless, the current research barely discusses the quality of solution predicted by the models, which should be an important indicator to check what GNNs learn from SAT problems. 
The end-to-end models mentioned above mainly focus on the prediction of satisfiability.
For a satisfiable instance, the predicted solution given by models is infeasible most of the time, and we also have no idea how far it is from a feasible solution.
In order to better understand the capability of GNN models, we target the subject as learning the solution of MaxSAT problem, an optimization version of SAT which aims at finding a solution that maximizes the number of satisfied clauses. The advantage is that as an optimization problem, we can analyze the quality of solution in more detail, such as the number of satisfied clauses and the approximation ratio, while a solution of SAT provides little information that it is valid or not. It is worth mentioning that although the 2SAT problem (i.e., each clause contains two literals) is in the complexity class P, the Max2SAT problem is still NP-hard \cite{garey1976max2sat}, so even generating near-optimal solution for Max2SAT is not trivial.

Moreover, despite GNNs have shown their effectiveness on SAT problem through experiments, the theoretical analysis about why they can work is still absent so far. Recently, there have been some theoretical research about learning to solve combinatorial problems with GNNs. \cite{xu2019gin} proves that the discriminative power of many popular GNN variants, such as Graph Convolutional Networks (GCN) and GraphSAGE, is not sufficient for the graph isomorphism problem. \cite{xu2020algoalign} proposes the algorithmic alignment theory to explain what reasoning tasks GNNs can learn well, and shows GNNs can align with dynamic programming algorithms. So, many classical reasoning tasks could be solved relatively well with a GNN model, such as visual question answering and shortest paths. However, it also raises an issue: as a polynomial-time procedure, any GNN model cannot align with an exact algorithm for NP-hard problems unless P=NP, whereas some approximation algorithms may be used to analyze the capability of GNNs to solve such problems. \cite{sato2019approximation} has employed a class of distributed local algorithms to prove that GNNs can learn to solve some combinatorial optimization problems, such as minimum dominating set and minimum vertex cover, with some approximation ratios. These theoretical contributions provide a basis for us to further study the capability of GNNs on the MaxSAT problem.

In this paper, we investigate the capability of GNNs in learning to solve MaxSAT problem both from theoretical and practical perspectives. Inspired from the algorithmic alignment theory, we design a distributed local algorithm for MaxSAT, which is guaranteed to have an approximation ratio of $1/2$. As the algorithm aligns well with a single-layer GNN, it can provide a theoretical explanation of the effect that a simple GNN can learn to solve MaxSAT problem to some extent.
More importantly, for the first time, this work leverages GNNs to learn the solution of MaxSAT problem. Specifically, we build two typical GNN models, MS-NSFG and MS-ESFG, based on summarizing the common forms proposed for SAT, to investigate the capability of GNNs from the practical perspective.
We have trained and tested both models on several random MaxSAT datasets in different settings. The experimental results demonstrate that both models have achieved pretty high accuracy on the testing sets, as well as the satisfactory generalization to larger and more difficult problems. This implies that GNNs are expected to be promising alternatives to improve the state-of-the-art solvers. We are hopeful that the results in this paper can provide preliminary knowledge about the capability of GNNs to solve MaxSAT problem, and become a basis for more theoretical and practical research in the future.

The contributions can be summarized as follows:
\begin{itemize}
	\item We build two typical GNN models to solve MaxSAT problem, which is the first work to predict the solution of MaxSAT problem with GNNs in an end-to-end fashion.
	\item We analyze the capability of GNNs to solve MaxSAT problem theoretically, and prove that even a single-layer GNN can achieve the approximation ratio of $1/2$.
	\item We evaluate the GNN models on several datasets of random MaxSAT instances with different distributions, and show that GNNs can achieve good performance and generalization on this task.
\end{itemize}

\begin{figure}[t]
	\centering
	\includegraphics[width=0.9\columnwidth]{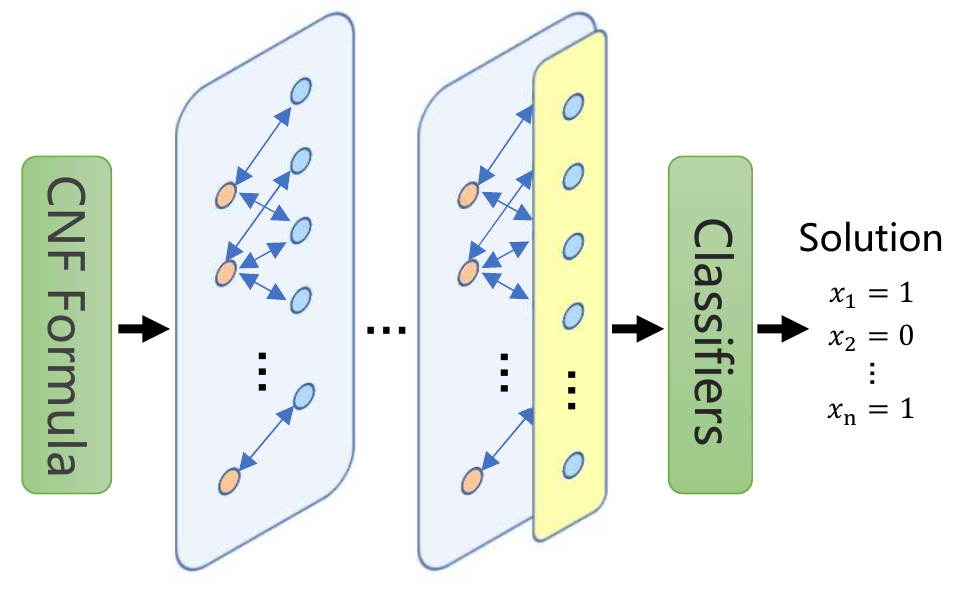} 
	\caption{The general framework of applying GNNs to solve SAT and MaxSAT problems. The CNF formula is firstly transformed to a bipartite factor graph, then each node is assigned an initial embedding. The embeddings are updated iteratively through the message passing process. Finally, a classifier decodes the assignment of each variable from its embedding.}
	\label{fig1}
\end{figure}

\section{Preliminaries}

In this section, we firstly give the definition of SAT and MaxSAT problems, and then give an introduction to graph neural networks.

\subsection{SAT and MaxSAT}
Given a set of Boolean variables $\{x_1, x_2, \dots, x_n\}$, a \emph{literal} is either a variable $x_i$ or its negation $\neg x_i$. A \emph{clause} is a disjunction of literals, such as $C_1:=x_1 \vee \neg x_2 \vee x_3$. A propositional logic formula in \emph{conjunctive normal form (CNF)} is a conjunction of clauses, such as $F:=C_1 \wedge C_2 \wedge C_3$. 

For a CNF formula $F$, the Boolean Satisfiability (SAT) problem is to find a truth assignment for each variable, such that all the clauses are satisfied (i.e., at least one literal in each clause is True). And the Maximum Satisfiability (MaxSAT) problem is to find a truth assignment for each variable, such that the number of satisfied clauses is maximized. SAT and MaxSAT have long been proved to be NP-hard, which means there is no polynomial-time algorithm to solve them unless P=NP. Specifically, if every clause in $F$ contains exactly $k$ literals, the problem is called `$k$SAT' or `Max$k$SAT'.

There is a close relationship between MaxSAT and SAT. Firstly, MaxSAT can be seen as an optimization version of SAT, which can provide more information when handling over-constrained problems. Besides, many solving techniques that are effective in SAT have also been adapted to MaxSAT, such as lazy data structures and variable selection heuristics. Accordingly, SAT can also benefit from the progress of MaxSAT.

\subsection{Graph Neural Networks}
Graph neural networks (GNNs) are a family of neural network architectures that operate on graphs, which have shown great power on many tasks across domains, such as protein interface prediction, recommendation systems and traffic prediction \cite{zhou2020gnnreview}. For a graph $G=\langle V,E \rangle$ where $V$ is a set of nodes and $E \subseteq V \times V$ is a set of edges, GNN models accept $G$ as input, and represent each node $v$ as an embedding vector $h_v^{(0)}$. Most popular GNN models follow the message passing process that updates the embedding of a node by aggregating the information of its neighbors iteratively. The operation of the $k$-th iteration (layer) of GNNs can be formalized as
\begin{equation}
\begin{array}{l}
s_v^{(k)}={\rm AGG}^{(k)}(\{h_u^{(k-1)}|u \in \mathcal{N}(v)\}), \\
h_v^{(k)}={\rm UPD}^{(k)}(h_v^{(k-1)}, s_v^{(k)}),
\end{array}
\label{equ1}
\end{equation}
where $h_v^{(k)}$ is the embedding vector of node $v$ after the $k$-th iteration. In the aggregating (messaging) step, a message $s_v^{(k)}$ is generated for each node $v$ by collecting the embeddings from its neighbors $\mathcal{N}(v)$. Next, in the updating (combining) step, the embedding of each node is updated combined with the message generated above. After $T$ iterations, the final embedding $h_v^{(T)}$ for each node $v$ is obtained. If the task GNNs need to handle is at the graph level such as graph classification, an embedding vector of the entire graph should be generated from that of all the nodes:
\begin{equation}
h_G={\rm READOUT}(\{h_u^{(T)}|u \in V\}).
\label{equ2}
\end{equation}
The final embeddings can be decoded into the outputs by a learnable function such as multi-layer perceptron (MLP), whether for node-level or graph-level tasks. The modern GNN variants have made different choices of aggregating, updating and readout functions, such as concatenation, summation, max-pooling and mean-pooling. A more comprehensive review of GNNs can be found in \cite{wuzh2021gnnreview}.

\begin{figure}[b]
	\centering
	\includegraphics[width=0.9\columnwidth]{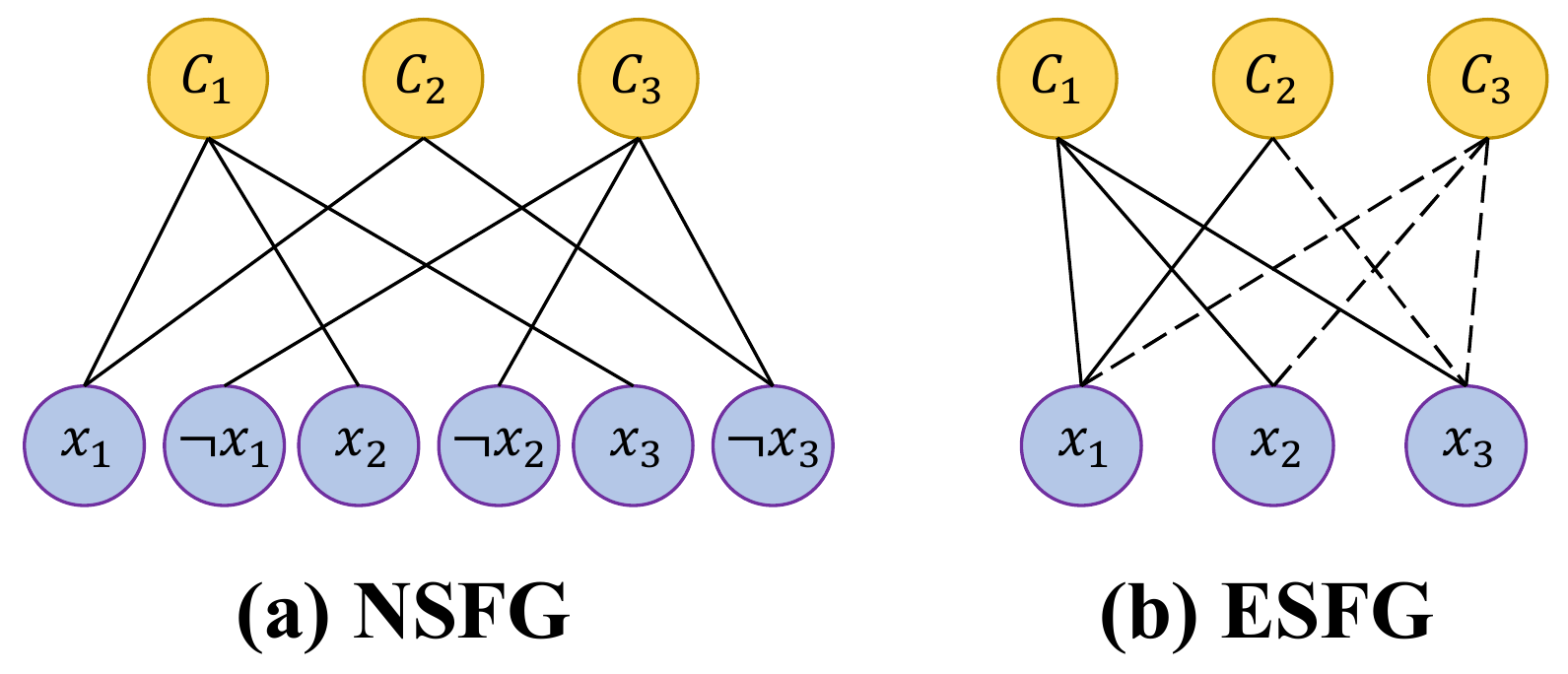} 
	\caption{Two kinds of factor graphs to represent the CNF formula $(x_1 \vee x_2 \vee x_3) \wedge (x_1 \vee \neg x_3) \wedge (\neg x_1 \vee \neg x_2 \vee \neg x_3)$ with 3 variables and 3 clauses.}
	\label{fig2}
\end{figure}

\section{GNN Models for MaxSAT}

As described in Section 1, there have been several attempts that learn to solve SAT problem with GNNs in recent years. The pipeline of such work generally consists of three parts. Firstly, a CNF formula is transformed to a graph through some rules. Next, a GNN variant is employed that maps the graph representation to the labels, e.g., the satisfiability of a problem or the assignment of a variable. Finally, the prediction results are further analyzed and utilized after the training converges. Intuitively, they can almost be transfered to work on MaxSAT problem without much modification, since these two problems have the same form of input.

Factor graph is a common representation for CNF formulas, which is a bipartite structure to represent the relationship between literals and clauses. There are mainly two kinds of factor graphs that have appeared in the previous work. The first one is node-splitting factor graph (NSFG), which splits the two literals ($x_i, \neg x_i$) corresponding to a variable $x_i$ into two nodes. NSFG is used by many work such as \cite{selsam2018learning} and \cite{zhang2020nlocalsat}. The other one is edge-splitting factor graph (ESFG), which establishes two types of edges, connected the clauses with positive and negative literals separately. \cite{nips2019learning} uses ESFG with a pair of biadjacency matrices. An example that represents a CNF formula with these two kinds of factor graphs is illustrated in Figure \ref{fig2}.

The models generally follow the message passing process of GNNs. Considering the bipartite structure of graph, in each layer the process is divided in two directions executed in sequence.
For NSFG, the operation of the $k$-th layer of GNNs can be formalized as
\begin{equation}
\begin{split}
&C_j^{(k)}={\rm UPD_C}(C_j^{(k-1)}, {\rm AGG_L}(L_i^{(k-1)}|(i,j) \in E)), \\
&(L_i^{(k)}, \widetilde{L}_i^{(k)})={\rm UPD_L}(L_i^{(k-1)}, \widetilde{L}_i^{(k-1)}, \\ 
&\mathrel{\phantom{(L_i^{(k)}, \widetilde{L}_i^{(k)})={\rm UPD_L}}}{\rm AGG_C}(C_j^{(k-1)}|(i,j) \in E)),
\end{split}
\label{equ3}
\end{equation}
where $L_i^{(k)}, C_j^{(k)}$ is the embedding of literal $i$ and clause $j$ in the $k$-th layer. The message is firstly collected from the literals in clause $j$, and the embedding of $j$ is updated by that of the last layer and the message. Next, the embedding of literal $i$ is updated in almost the same way, except that the embedding of literal $i$ is updated with that of its negation together (denoted by $\widetilde{L}_i^{(k)}$). This operation maintains the consistency of positive and negative literals of the same variable.

For ESFG, there are two types of edges, which can be denoted as $E^+$ and $E^-$. If a positive literal $i_1$ appears in clause $j$, we have $(i_1, j)\in E^+$, and similarly have $(i_2, j)\in E^-$ if $i_2$ is a negative literal.
The operation of the $k$-th layer of GNNs can be formalized as 
\begin{equation}
\begin{split}
&C_j^{(k)}={\rm UPD_C}(C_j^{(k-1)}, {\rm AGG_L^+}(L_i^{(k-1)}|(i,j) \in E^+), \\
&\mathrel{\phantom{C_j^{(k)}={\rm UPD_C}}}{\rm AGG_L^-}(L_i^{(k-1)}|(i,j) \in E^-)), \\
&L_i^{(k)}={\rm UPD_L}(L_i^{(k-1)}, {\rm AGG_C^+}(C_j^{(k-1)}|(i,j) \in E^+), \\
&\mathrel{\phantom{C_j^{(k)}={\rm UPD_C}}}{\rm AGG_C^-}(C_j^{(k-1)}|(i,j) \in E^-)).
\end{split}
\label{equ4}
\end{equation}

Finally, for both NSFG and ESFG, a binary classifier is applied to map the embedding of each variable to its assignment. We use the binary cross entropy (BCE) as loss function, which can be written as
\begin{equation}
BCE(y, p) = -(y\log(p)+(1-y)\log(1-p)),
\label{equ5}
\end{equation}
where $p\in[0,1]$ is the predicted probability of a variable being assigned True, and $y$ is the binary label from an optimal solution. By averaging the loss of each variable, we obtain the loss of a problem instance, which should be minimized.
We will investigate the performance of these two models through experiments in Section 5.

\section{Theoretical Analysis}

With the results and prospects revealed in practice, GNNs are considered to be a suitable choice in learning to solve SAT problem from benchmarks. However, the knowledge about why these models would work from a theoretical perspective is still limited so far. In this section, we are committed to explaining the possible mechanism of GNNs to solve MaxSAT problem. More specifically, we prove that a single-layer GNN can achieve an approximation ratio of $1/2$ with the help of a distributed local algorithm.

\subsection{Algorithmic Alignment}
Our theoretical analysis framework is inspired by the algorithmic alignment theory proposed by \cite{xu2020algoalign}, which provides a new point of view to understand the reasoning capability of neural networks. Formally, suppose a neural network $\mathcal{N}$ with $n$ modules $\mathcal{N}_i$, if by replacing each $\mathcal{N}_i$ with a function $f_i$ it can simulate, and $f_1,\dots,f_n$ generate a reasoning function $g$, we say $\mathcal{N}$ aligns with $g$.
As shown in that paper, even if different models such as MLPs, deep sets and GNNs have the same expressive power theoretically, GNNs tend to achieve better performance and generalization stably in the experiments. This can be explained by the fact that the computational structure of GNNs aligns well with the dynamic programming (DP) algorithms. Therefore, compared with other models, each component of GNNs only needs to approximate a simpler function, which means the algorithmic alignment can improve the sample complexity.

\subsection{Distributed Local Algorithm}
According to the algorithmic alignment theory, it is possible to reasonably infer the capability that a GNN model can achieve with the help of an algorithm aligned with its computational structure. However, this also raises an issue: as a polynomial-time procedure, any GNN model cannot align with an exact algorithm for NP-hard problems such as MaxSAT, under the assumption that P$\neq$NP. In this case, we turn to employ a weaker approximation algorithm to analyze the capability of GNNs. Although there has been a lot of research on the approximation algorithms of MaxSAT \cite{vijay2001approxbook}, it is regrettable that most of them cannot align with the structure of GNNs in an intuitive way.

In fact, there is a class of algorithms called distributed local algorithm (DLA) \cite{elkin2004distributed}, which have been found to be a good choice to align with GNNs for combinatorial problems. DLA assumes a distributed computing system that there are a set of nodes in a graph, where any two nodes do not know the existence of each other at first. The algorithm then runs in a constant number of synchronous communication rounds. In each round, a node performs local computation, and sends one message to its neighboring nodes, while receiving one from each of them. Finally, each node computes the output in terms of the information it holds. DLAs have been widely used in many applications, such as designing sublinear-time algorithms \cite{parnas2007sublinear} and controlling wireless sensor networks \cite{kubisch2003wireless}.

\subsection{Analyzing GNNs for MaxSAT}
Most DLAs are designed for solving combinatorial problems in graph theory. \cite{sato2019approximation} has employed DLAs to clarify the approximation ratios of GNNs for several NP-hard problems on graphs such as minimum dominating set and minimum vertex cover. This also encourages us to analyze the capability of GNNs to solve MaxSAT problem with the help of a DLA. As the DLA for MaxSAT has not been studied in the literature, we design an algorithm that aligns well with the message passing process described in Section 3, and the pseudo code is shown in Algorithm \ref{alg1}.

This algorithm accepts the description of a MaxSAT problem instance as input, and returns the assignment of literals, while the value of objective (i.e., the number of satisfied clauses) is easy to compute from it. Following the rules of DLA, the literals and clauses do not have any information about the problem initially.
$W$ and $S$ correspond to the states of literals and clauses, respectively. In line 4, the message is passed from \emph{literals} to \emph{clauses}, so that each clause is aware of the literals it contains, and stores the information into $S$. Next, in lines 5-8, the message is generated by \emph{clauses} and sent to \emph{literals}. Finally, the assignment of each literal is decoded from $W$ by a greedy-like method.
It can be seen that the algorithm only carries out one round of communication. It is interesting to know whether a single-layer GNN may also have some capabilities to solve the MaxSAT problem theoretically. This can be analyzed by aligning it with the algorithm.

\begin{algorithm}[tb]
	\caption{A distributed local algorithm for MaxSAT}
	\label{alg1}
	\textbf{Input}: The set of literals $L$, the set of clauses $C$ \\
	\textbf{Output}: The assignments of literals $\Phi$
	
	\begin{algorithmic}[1] 
		\STATE Set up a factor graph such as NSFG for $L$ and $C$.
		\STATE $W(L_i)\leftarrow 0$ for each $L_i \in L$.
		\STATE $S(C_j)\leftarrow \{\}$ for each $C_j \in C$.
		\STATE $S(C_j)\leftarrow S(C_j) \cup \{L_i\}$ for each edge $(L_i,C_j)$ in the factor graph.
		\FOR {each $C_j \in C$}
		\STATE $L^* \leftarrow$ Pick a literal from $S(C_j)$.
		\STATE $W(L^*)\leftarrow W(L^*)+1$.
		\ENDFOR
		\FOR {each $L_i \in L$}
		\IF {$W(L_i) >= W(\widetilde{L}_i)$}
		\STATE $\Phi(L_i)\leftarrow {\rm True}$.
		\ELSE
		\STATE $\Phi(L_i)\leftarrow {\rm False}$.
		\ENDIF
		\ENDFOR
		\STATE \textbf{return} $\Phi$
	\end{algorithmic}
\end{algorithm}

\begin{theorem}
	\label{theo1}
	There exists a single-layer GNN to solve the MaxSAT problem, which is guaranteed to have an approximation ratio of\, $1/2$. 
\end{theorem}

\begin{proof}
	We give a proof sketch here because of the space limitation, and the full proof is available in Appendix.
	Firstly, we show that there exists a single-layer GNN model that can align with the proposed distributed local algorithm (Algorithm \ref{alg1}), so that every part of the algorithm can be effectively approximated by a component of GNN. This implies that the GNN model can achieve the same performance with the algorithm when solving MaxSAT problem.
	Next, we prove that it is a $1/2$-approximation algorithm. The picking operation of literal (line 6) is equivalent to transforming the original problem into another, where each clause only contains one literal. Given a solution, the number of satisfied clauses of the original problem is at least as large as that of the new one, so we get the approximation ratio by analyzing the new Max1SAT problem. It is not hard to find that the algorithm has a guarantee that at least half of the clauses are satisfied for any Max1SAT instance.
\end{proof}

The proposed DLA can be considered as a candidate explanation for why a single-layer GNN model works and generalizes on the MaxSAT instances.
The methodologies and results may also serve as a basis for future work to make theoretical progress on this task, such as explaining the improvement of capability as the number of GNN layers increases in the experiments and finding a tighter bound.

%

\section{Experimental Evaluation}

Although the GNN models have shown their capability on many reasoning tasks such as SAT problem, the experimental evidence of whether GNNs can learn to solve MaxSAT problem is still under exploration. In order to demonstrate the capability of GNNs on this task, we firstly build two models, using NSFG and ESFG separately. After constructing the datasets with a commonly used generator for random MaxSAT instances, the GNN models are trained and tested in different settings, so that we can obtain a comprehensive understanding of their capabilities. The experimental results show that GNNs have attractive potential in learning to solve MaxSAT problem with good performance and generalization.

\subsection{Building the Models}

In Section 3, we have described the general GNN frameworks that can learn to solve MaxSAT problem. 
Based on the successful early attempts, we build two GNN models that accept a CNF formula as input, and output the assignment of literals. We abbreviate them to MS-NSFG and MS-ESFG, because they use NSFG and ESFG as the graph representation, respectively. The structures of models follow those in Eq. (\ref{equ3}) and (\ref{equ4}). Here we only describe the implementation of some key components, and the complete frameworks are shown in Appendix.

\noindent \textbf{Aggregating functions.} The implementation of aggregating functions: ${\rm AGG_L}$, ${\rm AGG_C}$ in MS-NSFG, and ${\rm AGG_L^+}$, ${\rm AGG_L^-}$, ${\rm AGG_C^+}$, ${\rm AGG_C^-}$ in MS-ESFG, consists of two steps. First, an MLP module maps the embedding of each node to a message vector. After that, the messages from relevant nodes are summed up to get the final output. For example, the function ${\rm AGG_L}$ in MS-NSFG which generates the message from literals and sends it to clause $j$, can be formalized as $\sum_{(i, j)\in E}{{\rm MLP}(L_i)}$, where $L_i$ is the embedding of literal $i$.

\noindent \textbf{Updating functions.} The updating functions: ${\rm UPD_L}$ and ${\rm UPD_C}$ in both MS-NSFG and MS-ESFG can be implemented by the LSTM module \cite{hochreiter1997lstm}. For example, to implement the function ${\rm UPD_C}$, the embedding of clause $j$ (denoted as $C_j$) is taken as the hidden state, and the message generated from aggregation is taken as the input of LSTM.

\subsection{Data Generation}
For data-driven approaches, a large number of labeled instances are necessary. So, the dataset should be easy to solve by off-the-shelf MaxSAT solvers to ensure the size of dataset, meanwhile it should also have the ability to produce larger instances in the same distribution to test the generalization. As a result, we construct the datasets of random MaxSAT instances by running a generator proposed by \cite{mitchell1992generator}, which is also used to provide benchmarks for the random track of MaxSAT competitions\footnote{https://maxsat-evaluations.github.io/}. The generator can produce CNF formulas with three parameters: the number of literals in each clause $k$, the number of variables $n$, and the number of clauses $m$.

In order to better observe the performance of GNN models, we generate multiple datasets with different distributions, which are listed in Table \ref{tab1}.
For \texttt{R2(60,600)}, \texttt{R2(60,800)} and \texttt{R3(30,300)}, we generate 20K instances, and divide them into training set, validation set and testing set according to the ratio of 8:1:1. For \texttt{R2(80,800)} and \texttt{R3(50,500)}, we only generate 2K instances as testing sets. Then we call MaxHS \cite{bacchus2020maxhs}, a state-of-the-art MaxSAT solver, to find an optimal solution for each instance as the labels of variables.
The average time MaxHS spends to solve the instances in each dataset is also presented, so that we can know about their difficulty.

\begin{table}[h]
	\def\arraystretch{1.1}
	\centering
	\begin{tabular}{ccccc}
		\hlineB{3}
		Dataset & $k$ & $n$ & $m$ & Time(s) \\ \hline
		\texttt{R2(60,600)}   & 2 & 60 & 600 & 4.01  \\
		\texttt{R2(60,800)}   & 2 & 60 & 800 & 18.21  \\
		\texttt{R2(80,800)}   & 2 & 80 & 800 & 105.19  \\
		\texttt{R3(30,300)}   & 3 & 30 & 300 & 1.92  \\
		\texttt{R3(50,500)}   & 3 & 50 & 500 & 216.26  \\ \hlineB{3}
	\end{tabular}
	\caption{The parameters and difficulty of the datasets.}
	\label{tab1}
\end{table}

\subsection{Implementation Details}
We implement the models MS-NSFG and MS-ESFG in Python, and examine their practical performance to solve MaxSAT problem\footnote{The source code is available at https://url-is-hidden.}. The models are trained by the Adam optimizer \cite{kingma2015adam}. All the experiments are running on a machine with Intel Core i7-8700 CPU (3.20GHz) and NVIDIA Tesla V100 GPU.

For reproducibility, we also summarize the setting of hyper-parameters as follows. In our configuration, the dimension of embeddings and messages $d=128$, and the learning rate is $2 \times 10^{-5}$ with a weight decay of $10^{-10}$. Unless otherwise specified, the number of GNN layers $T=20$. The instances are fed into models in batches, with each batch containing 20K nodes.

\subsection{Accuracy of Models}
We firstly evaluate the performance of both GNN models, including their convergence and the quality of predicted solution. The accuracy of prediction is reported with two values in the following paragraphs. The first one is \textbf{the gap to optimal objective}, which is the distance between the predicted objective and the corresponding optimal value. The predicted objective can be computed by counting the satisfied clauses given the predict solution. The other is \textbf{the accuracy of assignments}, which is the percentage of correctly classified variables, i.e., the assignment of a variable from the predicted solution is the same as its label. Generally, the gap to optimal objective could be a better indicator, because the optimal solution of a MaxSAT problem may not be unique.

We have trained MS-NSFG and MS-ESFG separately on three different datasets: \texttt{R2(60,600)}, \texttt{R2(60,800)} and \texttt{R3(30,300)}, and illustrate the evolution curves of accuracy throughout training on \texttt{R2(60,600)} in Figure \ref{fig3} as an example. All the models can converge within 150 epochs, and achieve pretty good performance. The average gaps to optimal objectives are less than 2 clauses for all the models, with the approximation ratio $>$99.5\%. Besides, the accuracy of assignments is around 92\% (Max2SAT) and 83\% (Max3SAT). 
There is no significant difference between the two models, while MS-ESFG performs slightly better than MS-NSFG. The experimental results show that both GNN models can be used in learning to solve MaxSAT problem.

\begin{figure}[htbp]
	\centering
	\includegraphics[width=0.98\columnwidth]{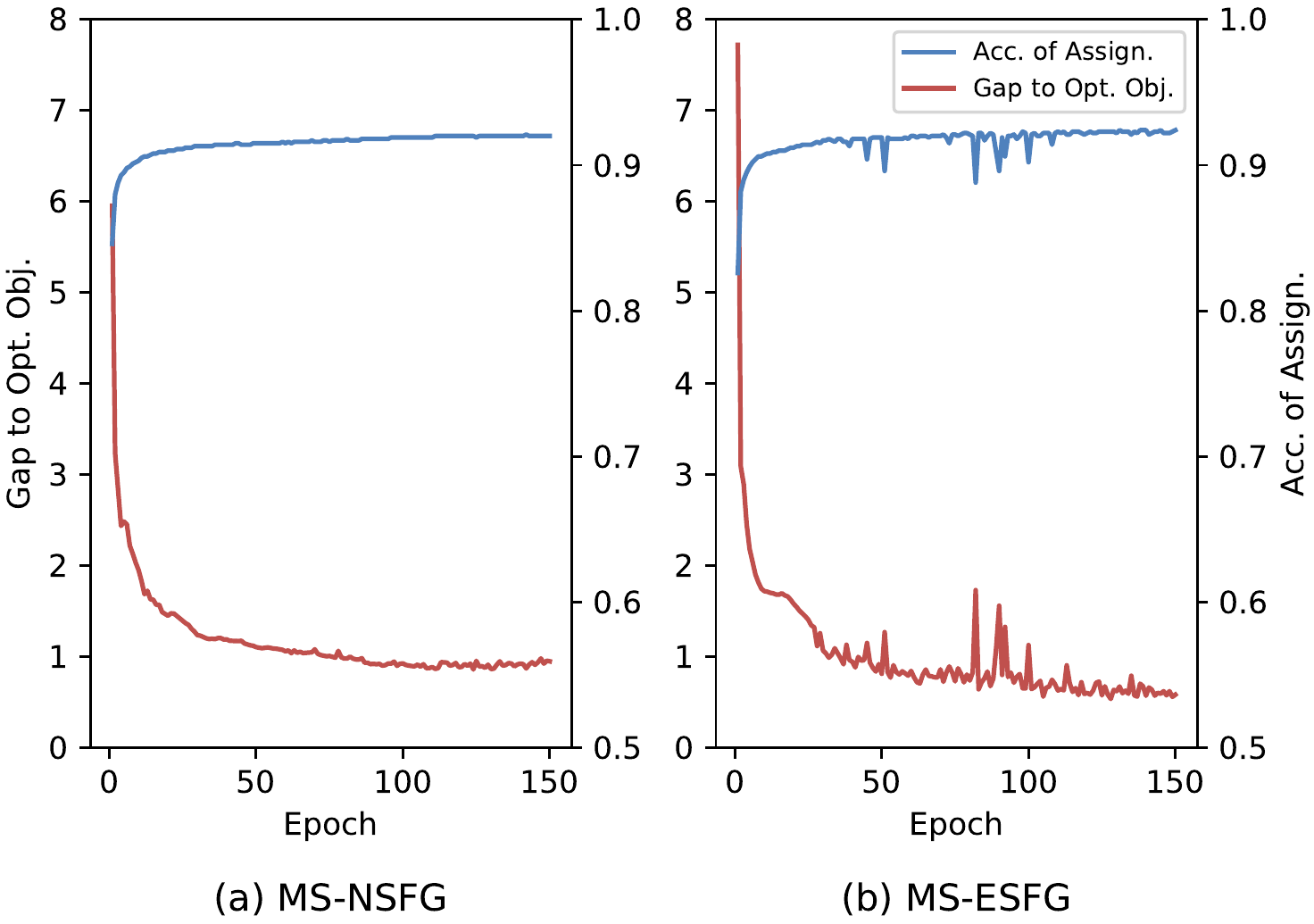} 
	\caption{The evolution of accuracy of MS-NSFG and MS-ESFG during a training process of 150 epochs on the dataset \texttt{R2(60,600)}.}
	\label{fig3}
\end{figure}

\subsection{Influence of GNN Layers}
According to the theoretical analysis, the number of GNN layers determines the information that each node can receive from its neighborhood, thereby affecting the accuracy of prediction. We examine this phenomenon from the experimental perspective by training MS-NSFG and MS-ESFG on the three datasets separately, with different hyper-parameter $T$ from 1 to 30. The changes of accuracy on testing sets are illustrated in Figure \ref{fig4}. It can be seen that when $T=1$, the capabilities of both models are weak, where the predicted objective is far from the optimal value, and the accuracy of assignments is not ideal. However, the effectiveness of both models has been significantly improved when $T=5$. If we increase $T$ to 20, the number of clauses satisfied by the predicted solution is only one less than the optimal value on average. We have continued to increase $T$ to 60, but no obvious improvement occurred. This indicates that increasing the number of layers -- within an appropriate range -- will improve the capability of GNN models.

\begin{figure}[htbp]
	\centering
	\includegraphics[width=0.98\columnwidth]{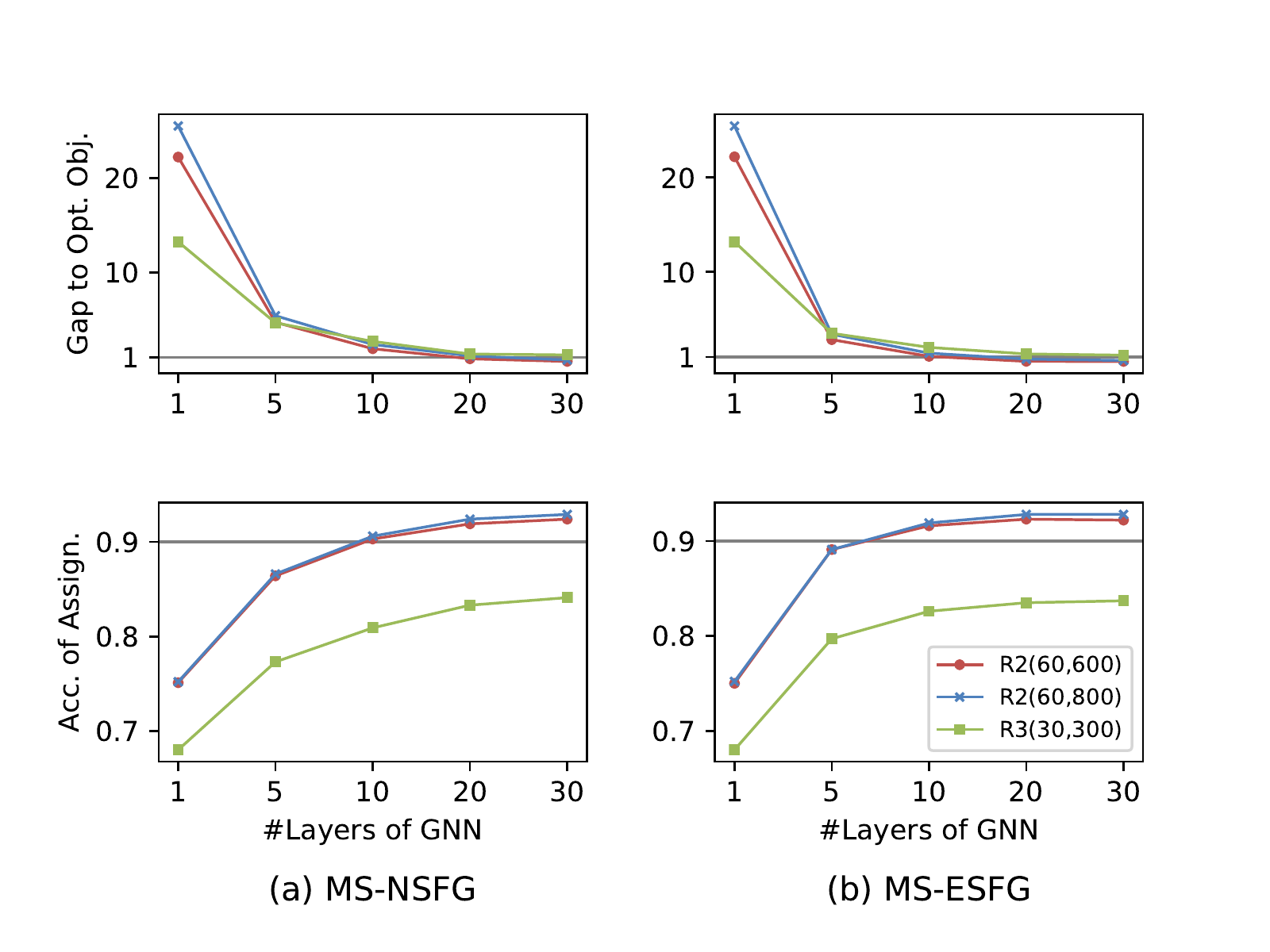} 
	\caption{The accuracy of MS-NSFG and MS-ESFG trained with different number of GNN layers.}
	\label{fig4}
\end{figure}

\subsection{Generalizing to Other Distributions}
Generalization is an important factor in evaluating the possibility to apply GNN models to solve those harder problems. We make the predictions by MS-NSFG and MS-ESFG on the testing sets with different distributions from the training sets, and the results are shown in Table \ref{tab1} and \ref{tab2}, respectively. The first column is the dataset used for training the model, and the first row is the testing set. For each pair of training and testing datasets, we report the gap to optimal objective together with the approximation ratio in the brackets above, and the accuracy of assignments below. Here, the approximation ratio is defined as the ratio of predicted and optimal objectives.

Here we mainly focus on three kinds of generalization. The first is generalizing to the datasets with \textbf{different clause-variable proportion}. From the results, both models trained on \texttt{R2(60,600)} and tested on \texttt{R2(60,800)} (or vice versa) can maintain almost the same prediction accuracy. The second is generalizing to \textbf{larger and more difficult problems}. We use two testing sets, \texttt{R2(80,800)} and \texttt{R3(50,500)}, where the number of variables is larger than those appeared in the training sets, as well as more difficult to solve by MaxHS. It can be found that both models trained on Max2SAT datasets can generalize to work on \texttt{R2(80,800)} with a satisfactory accuracy. This also holds when the training set is \texttt{R3(30,300)} and the testing set is \texttt{R3(50,500)}, which implies that GNN models are expected to be promising alternatives to help solve those difficult MaxSAT problems. The last is generalizing to \textbf{other datasets with different parameter $k$}. For example, the model is trained on Max2SAT but tested on Max3SAT problems. The results show that both models have limitations under this condition, since they cannot achieve an accuracy of assignments $>$80\% on every pair of training and testing sets, and the gap to optimal objective is not close enough, especially when trained on \texttt{R3(30,300)} and tested on Max2SAT datasets.

\begin{table}[t]
	\centering
	\def\arraystretch{1.1}
	\resizebox{.95\columnwidth}{!}{
		\begin{tabular}{c|ccc|cc}
			\hlineB{3}
			\rule{0pt}{10pt} Train $\backslash$ Test & \texttt{R2(60,600)} & \texttt{R2(60,800)} & \texttt{R3(30,300)} & \texttt{R2(80,800)} & \texttt{R3(50,500)} \\ \hlineB{2}
			\multirow{2}{*}{\texttt{R2(60,600)}} & \textbf{0.86 (99.8\%)}  & 1.42 (99.7\%)  & 3.86 (98.6\%)  & \textbf{1.15 (99.8\%)} & 6.77 (98.5\%) \\
			& \textbf{91.9\%}  & 91.9\%  & 76.1\%  & \textbf{92.0\%} & 75.4\% \\ \hline
			\multirow{2}{*}{\texttt{R2(60,800)}} & 1.19 (99.7\%)  & \textbf{1.17 (99.8\%)}  & 4.59 (98.4\%)  & 1.50 (99.7\%) & 7.96 (98.3\%) \\
			& 91.7\%  & \textbf{92.4\%}  & 75.8\%  & 91.6\% & 75.1\% \\ \hline
			\multirow{2}{*}{\texttt{R3(30,300)}} & 20.58 (96.0\%)  & 29.10 (95.7\%) & \textbf{1.37 (99.5\%)}  & 27.92 (95.9\%) & \textbf{2.16 (99.5\%)} \\
			& 78.1\%  & 76.0\%  & \textbf{83.3\%}  & 78.0\% & \textbf{82.2\%} \\ \hlineB{3} 
	\end{tabular}}
	\caption{The accuracy of MS-NSFG on different combinations of training and testing sets.}
	\label{tab2}
\end{table}

\begin{table}[t]
	\centering
	\def\arraystretch{1.1}
	\resizebox{.95\columnwidth}{!}{
		\begin{tabular}{c|ccc|cc}
			\hlineB{3}
			\rule{0pt}{10pt} Train $\backslash$ Test & \texttt{R2(60,600)} & \texttt{R2(60,800)} & \texttt{R3(30,300)} & \texttt{R2(80,800)} & \texttt{R3(50,500)} \\ \hlineB{2}
			\multirow{2}{*}{\texttt{R2(60,600)}} & 0.54 (99.8\%)  & 1.12 (99.8\%)  & 5.98 (97.9\%)  & 0.69 (99.9\%) & 9.92 (97.9\%) \\
			& 92.3\%  & 92.2\%  & 74.5\%  & 92.2\% & 73.9\% \\ \hline
			\multirow{2}{*}{\texttt{R2(60,800)}} & \textbf{0.48 (99.9\%)}  & \textbf{0.78 (99.8\%)}  & 6.13 (97.8\%)  & \textbf{0.62 (99.9\%)} & 10.24 (97.8\%) \\
			& \textbf{92.3\%}  & \textbf{92.8\%}  & 74.7\%  & \textbf{92.2\%} & 74.2\% \\ \hline
			\multirow{2}{*}{\texttt{R3(30,300)}} & 14.83 (97.1\%)  & 16.44 (97.5\%) & \textbf{1.32 (99.5\%)}  & 20.22 (97.0\%) & \textbf{1.98 (99.5\%)} \\
			& 78.8\%  & 79.5\%  & \textbf{83.5\%}  & 78.6\% & \textbf{82.5\%} \\ \hlineB{3} 
	\end{tabular}}
	\caption{The accuracy of MS-ESFG on different combinations of training and testing sets.}
	\label{tab3}
\end{table}

\section{Related Work}

Although the mainstream approaches for solving combinatorial problems, not limited to SAT or MaxSAT, are based on the search algorithms from symbolism, there have always been attempts trying to tackle these problems through data-driven techniques. A class of research is to integrate machine learning model in the traditional search framework, which has made progress on a number of problems such as mixed integer programming (MIP) \cite{khalil2016mip}, satisfiability modulo theories (SMT) \cite{balunovic2018smt} and quantified boolean formulas (QBF) \cite{lederman2020qbflearning}. Here, we work on building end-to-end models which do not need the aid of search algorithm. The earliest research work can be traced back to the Hopfield network \cite{hopfield1985neural} to solve TSP problem. Recently, many variants of neural networks have been proposed, which directly learn to solve combinatorial problems. \cite{vinyals2015pointer} introduces Pointer Net, a sequential model that performs well on solving TSP and convex hull problems. \cite{khalil2017learning} uses the combination of reinforcement learning and graph embedding, and learns greedy-like strategies for minimum vertex cover, maximum cut and TSP problems.

With the development of graph neural networks, there have been a series of work that uses GNN models to solve combinatorial problems. An important reason is that many of such problems are directly defined on graph, and in general the relation between variables and constraints can be naturally represented as a bipartite graph. Except for the mentioned NeuroSAT \cite{selsam2018learning} and its improvement, there have been some efforts learning to solve TSP \cite{prates2019tsp}, pseudo-Boolean \cite{liu2020pseudo} and graph coloring \cite{lemos2019graphcolor} problems with GNN-based models. The results indicate that GNNs have become increasingly appealing alternatives in solving combinatorial problems. Moreover, there are also some well-organized literature reviews on this subject, such as \cite{bengio2021ml4co}, \cite{cappart2021corgnn} and \cite{lamb2020gnn4ns}.

\section{Conclusion and Future Work}

Graph neural networks (GNNs) have been considered as a promising technique that can learn to solve combinatorial problems in the data-driven fashion, especially for the Boolean Satisfiability (SAT) problem. In this paper, we further study the quality of solution predicted by GNNs in learning to solve Maximum Satisfiability (MaxSAT) problem, both from theoretical and practical perspectives. Based on the graph construction methods in the previous work, we build two kinds of GNN models, MS-NSFG and MS-ESFG, which can predict the solution of MaxSAT problem. The models are trained and tested on randomly generated benchmarks with different distributions. The experimental results show that both models have achieved pretty high accuracy, and also satisfactory generalization to larger and more difficult instances.
In addition, this paper is the first that attempts to present an explanation of the capability of GNNs to solve MaxSAT problem from a theoretical point of view. On the basis of algorithmic alignment theory, we prove that even a single-layer GNN model can solve the MaxSAT problem with an approximation ratio of $1/2$.

We hope the results in this paper can inspire future work from multiple perspectives. A promising direction is to integrate the GNN models into a powerful search framework to handle more difficult problems in the specific domain, such as weighted and partial MaxSAT problems.
It is also interesting to further analyze the theoretical capability of multi-layer GNNs to achieve better approximation.

\bibliography{aaai22}

\newpage
\appendix

\section*{Appendix}
\bigskip
\subsection{Proof of Theorem 1}

A CNF formula $\phi$ can be represented as a pair $(L,C)$, where $L$ is the set of literals, and $C$ is the set of clauses. Let $\mathcal{A}$ be the proposed distributed local algorithm for MaxSAT problem.
Fist of all, we present the following two lemmas:

\begin{lemma}
\label{lem1}
Given the algorithm $\mathcal{A}$, there exists a single-layer graph neural network $\mathcal{N}$, such that for any input $\phi=(L,C)$, $\mathcal{A}(\phi)=\mathcal{N}(\phi)$ holds.
\end{lemma}

\begin{lemma}
\label{lem2}
The algorithm $\mathcal{A}$ has an approximation ratio of $1/2$ for any input $\phi=(L,C)$.
\end{lemma}

If these lemmas hold, we have found a single-layer GNN $\mathcal{N}$ that achieves the same performance as $\mathcal{A}$ when solving MaxSAT problem, which is guaranteed to have a $1/2$-approximation. As a consequence, Theorem 1 holds.

\begin{proof}[Proof of Lemma \ref{lem1}]
We use a single-layer GNN $\mathcal{N}$ to align with the algorithm $\mathcal{A}$. Let $L$ be a finite set of literals, and $d=|L|$. We assume that the embedding of each clause is a 0-1 vector of length $d$, which represents a set of literals, and the embedding of each literal $L_i$ is composed of two values: $W(L_i)$ and $W(\widetilde{L}_i)$. Then, we construct some key components of $\mathcal{N}$ to align with the operations in $\mathcal{A}$.

Consider the operation $S(C_j)\leftarrow S(C_j) \cup \{L_i\}$ (line 4). As $S(C_j)$ and $\{L_i\}$ are sets of literals, this set union operation is equivalent to a logical AND function of two 0-1 vectors, which is apparently linearly separable. So there exists a learnable function $f_1: \mathbb{R}^{d} \times \mathbb{R}^d \rightarrow \mathbb{R}^d$ to simulate the operation exactly.

Consider the operation of picking a literal $L^*$ from $S(C_j)$, and $W(L^*)\leftarrow W(L^*)+1$ (lines 6-7). According to the universal approximation theorem, there exists a learnable function $f_2: \mathbb{R} \times \mathbb{R}^{d} \rightarrow \mathbb{R}$ to approximate this operation with arbitrarily small error $\epsilon$. Next, we show the error will not break the exact simulation of the assignment operation (lines 10-14). When $W(L_i) \ne W(\widetilde{L}_i)$, the condition $W(L_i) \ge W(\widetilde{L}_i)$ still holds if $\epsilon<0.5$, since the elements in $W$ are integers. Besides, when $W(L_i)=W(\widetilde{L}_i)$, either $L_i$ or $\widetilde{L}_i$ can be assigned True. So there exists an $f_2$ with $\epsilon<0.5$ to simulate the operations exactly.

The alignment and simulation of other parts are straightforward. As shown above, each component of $\mathcal{N}$ can align with an operation in $\mathcal{A}$. Therefore, for any input $\phi=(L,C)$, their outputs must be equal, i.e., $\mathcal{A}(\phi)=\mathcal{N}(\phi)$.
\end{proof}

\begin{proof}[Proof of Lemma \ref{lem2}]
In the algorithm $\mathcal{A}$, the picking operation of literal (line 6) transforms the original MaxSAT problem \emph{$P$} into a new Max1SAT problem \emph{$P'$}, where $W(L_i)$ counts the number of occurrences of literal $L_i$ in \emph{$P'$}. Given a solution of \emph{$P'$}, the number of satisfied clauses of \emph{$P$} is at least as large as that of \emph{$P'$}. Then, we consider the assignment operation of literal (lines 9-15) in $\mathcal{A}$. For a literal $L_i$ appearing in \emph{$P'$}, according to the condition $W(L_i) \ge W(\widetilde{L}_i)$, the satisfied clauses are no less than the rejected ones if $L_i$ is assigned True. Therefore, a solution produced by $\mathcal{A}$ satisfies at least $|C|/2$ clauses, which implies an approximation ratio of $1/2$.

\end{proof}

\subsection{Frameworks of Models}

The frameworks of MS-NSFG and MS-ESFG are shown below. The embeddings of literals and clauses in the $k$-th layer are denoted as $L^{(k)}$ and $C^{(k)}$, respectively. $d$ is the dimension of embeddings, and $T$ is the number of GNN layers.

\begin{algorithm}[htb]
	\caption{MS-NSFG}
	\label{msnsfg}
	\textbf{Input}: An NSFG $G=\langle V_L,V_C,E \rangle$ \\
	\textbf{Output}: The predicted assignments of literals $\Phi$
	
	\begin{algorithmic}[1] 
		\STATE Build the adjacency matrix $M$ of graph $G$;
		\STATE Initialize $L^{(0)} \in \mathbb{R}^{|V_L|\times d} \sim U(0,1)$;
		\STATE Initialize $C^{(0)} \in \mathbb{R}^{|V_C|\times d} \sim U(0,1)$;
		\FOR {$k$ from 1 to $T$}
		\STATE $C^{(k)}\leftarrow{\rm UPD_C}(C^{(k-1)}, M\cdot{\rm AGG_L}(L^{(k-1)}))$;
		\STATE $L^{(k)}\leftarrow{\rm UPD_L}(L^{(k-1)}, M^{\mathsf{T}}\cdot{\rm AGG_C}(C^{(k-1)}))$;
		\ENDFOR
		\STATE $L_{pred} \in \mathbb{R}^{|V_L|} \leftarrow {\rm PRED_L}(L^{(T)})$;
		\STATE $\Phi\leftarrow {\rm Round}({\rm Sigmoid}(L_{pred}))$;
		\STATE \textbf{return} $\Phi$
	\end{algorithmic}
\end{algorithm}

\begin{algorithm}[htb]
	\caption{MS-ESFG}
	\label{msesfg}
	\textbf{Input}: An ESFG $G=\langle V_L,V_C,E^+,E^- \rangle$ \\
	\textbf{Output}: The predicted assignments of literals $\Phi$
	
	\begin{algorithmic}[1] 
		\STATE Build the adjacency matrices $M^+,M^-$ of graph $G$;
		\STATE Initialize $L^{(0)} \in \mathbb{R}^{|V_L|\times d} \sim U(0,1)$;
		\STATE Initialize $C^{(0)} \in \mathbb{R}^{|V_C|\times d} \sim U(0,1)$;
		\FOR {$k$ from 1 to $T$}
		\STATE $C^{(k)}\leftarrow{\rm UPD_C}(C^{(k-1)},\; M^+ \cdot {\rm AGG_L^+}(L^{(k-1)}) + M^- \cdot {\rm AGG_L^-}(L^{(k-1)}))$;
		\STATE $L^{(k)}\leftarrow{\rm UPD_L}(L^{(k-1)},\; (M^+)^{\mathsf{T}} \cdot {\rm AGG_C^+}(C^{(k-1)}) + (M^-)^{\mathsf{T}} \cdot {\rm AGG_C^-}(C^{(k-1)}))$;
		\ENDFOR
		\STATE $L_{pred} \in \mathbb{R}^{|V_L|} \leftarrow {\rm PRED_L}(L^{(T)})$;
		\STATE $\Phi\leftarrow {\rm Round}({\rm Sigmoid}(L_{pred}))$;
		\STATE \textbf{return} $\Phi$
	\end{algorithmic}
\end{algorithm}

\end{document}